\documentclass[letterpaper, 10 pt, conference]{ieeeconf}  % Comment this line out
\IEEEoverridecommandlockouts                              % This command is only
\usepackage{graphics} % for pdf, bitmapped graphics files
\usepackage{epsfig} % for postscript graphics files
\usepackage{times} % assumes new font selection scheme installed
\usepackage{booktabs}      % \toprule, \midrule, \bottomrule
\usepackage{threeparttable} % \begin{threeparttable}, \tnote
\usepackage{adjustbox}     % \begin{adjustbox}{max width=...}
\usepackage{array}         % better column formatting

\usepackage{amsmath} % assumes amsmath package installed
\usepackage{amssymb}  % assumes amsmath package installed
\usepackage[hidelinks]{hyperref}
\usepackage{amsthm} % better to load after ams math
\usepackage{xfrac}
\usepackage{algorithmicx}
\usepackage{algpseudocode}
\usepackage{algorithm}
\usepackage{bm}
\usepackage{xspace}
\usepackage{nicefrac}
\usepackage{pifont}
\usepackage[capitalise]{cleveref}
\usepackage[dvipsnames]{xcolor}
\usepackage{dsfont}

\newcommand{\ignore}[1]{}

\newtheorem{theorem}{Theorem}
\newtheorem{lemma}{Lemma}
\newtheorem{assumption}{Assumption}

\theoremstyle{definition}

\theoremstyle{remark}

\newcommand{\hallucinated}[1]{{#1}}

\newcommand{\Rdyn}{{R}}

\usepackage{xifthen}

\newcommand{\likelihood}[2][]{%
    \ifthenelse{\isempty{#1}}
        {p\left(#2\right)}
        {p_{#1}\left(#2\right)}%
}

\newcommand{\normal}[3][]{%
    \ifthenelse{\isempty{#1}}
        {\mathcal{N}\left(#2, #3\right)}
        {\mathcal{N}\left(#1|#2, #3\right)}%
}

\newcommand{\defeq}{\overset{\text{def}}{=}}

\newcommand{\norm}[1]{\left\lVert#1\right\rVert}

\newcommand{\set}[1]{\left\{#1\right\}}

\RenewDocumentCommand{\H}{mo}{\mathrm{H}\IfValueTF{#2}{\left[#1\ \middle|\ #2\right]}{\parentheses*{#1}}}
\NewDocumentCommand{\Hsm}{mo}{\mathrm{H}\IfValueTF{#2}{[#1 \mid #2]}{\parentheses{#1}}}
\NewDocumentCommand{\I}{mmo}{\mathrm{I}\IfValueTF{#3}{\!\left(#1;#2\ \middle|\ #3\right)}{\parentheses*{#1; #2}}}
\NewDocumentCommand{\Ism}{mmo}{\mathrm{I}\IfValueTF{#3}{(#1;#2 \mid #3)}{\parentheses{#1; #2}}}

\newcommand{\R}{\mathbb{R}}

\newcommand{\N}{\mathbb{N}}

\newcommand{\E}{\mathbb{E}}

\newcommand{\Hspace}{\mathcal{H}_{k,B}^{d_x}}
\newcommand{\normHspace}[1]{\left\lVert #1 \right\rVert_{\mathcal{H}_{k}}}

\def\mI{{\bm{I}}}

\def\mK{{\bm{K}}}

\def\vzero{{\bm{0}}}

\def\vmu{{\bm{\mu}}}
\def\vtheta{{\bm{\theta}}}

\def\vf{{\bm{f}}}

\def\vk{{\bm{k}}}

\def\vu{{\bm{u}}}

\def\vw{{\bm{w}}}
\def\vx{{\bm{x}}}

\def\vy{{\bm{y}}}
\def\vz{{\bm{z}}}
\def\vpi{{\bm{\pi}}}
\def\vmu{{\bm{\mu}}}
\def\vsigma{{\bm{\sigma}}}

\def\vtau{{\bm{\tau}}}

\def\setA{{\mathcal{A}}}

\def\setD{{\mathcal{D}}}

\def\setH{{\mathcal{H}}}

\def\setN{{\mathcal{N}}}
\def\setO{{\mathcal{O}}}

\def\setU{{\mathcal{U}}}

\def\setX{{\mathcal{X}}}

\def\setZ{{\mathcal{Z}}}

\usepackage{cite}
\newcommand{\citep}[1]{\cite{#1}}
\usepackage{titletoc}
\newif\ifpreprint
\preprinttrue

\usepackage{bibunits}
\defaultbibliographystyle{IEEEtran}
\defaultbibliography{references}

\title{\LARGE \bf
Model-Based Reinforcement Learning for Control\\under Time-Varying Dynamics
}

\author{Klemens Iten, Bruce Lee, Chenhao Li, Lenart Treven, Andreas Krause, Bhavya Sukhija% <-this % stops a space
\thanks{All authors are with ETH Zürich,
        R\"amistrasse 101, 8092 Z\"urich, Switzerland
        {\tt\small \{kiten,trevenl,krausea,sukhijab\}@ethz.ch}}%
\thanks{B.~Lee, C.~Li, and A.~Krause are with the ETH AI Center at ETH Z\"urich
        {\tt\small \{bruce.lee,chenhao.li\}@ai.ethz.ch}}%
\ifpreprint
\thanks{This work has been submitted to the IEEE for possible publication. Copyright may be transferred without notice, after which this version may no longer be accessible.}
\fi
}

\begin{document}
\begin{bibunit}

\maketitle
\thispagestyle{empty}
\pagestyle{empty}

%%%%%%%%%%%%%%%%%%%%%%%%%%%%%%%%%%%%%%%%%%%%%%%%%%%%%%%%%%%%%%%%%%%%%%%%%%%%%%%%
\begin{abstract}

Learning-based control methods typically assume stationary system dynamics, an assumption often violated in real-world systems due to drift, wear, or changing operating conditions. We study reinforcement learning for control under time-varying dynamics. We consider a continual model-based reinforcement learning setting in which an agent repeatedly learns and controls a dynamical system whose transition dynamics evolve across episodes. We analyze the problem using Gaussian process dynamics models under frequentist variation-budget assumptions. Our analysis shows that persistent non-stationarity requires explicitly limiting the influence of outdated data to maintain calibrated uncertainty and meaningful dynamic regret guarantees. Motivated by these insights, we propose a practical optimistic model-based reinforcement learning algorithm with adaptive data buffer mechanisms and demonstrate improved performance on continuous control benchmarks with non-stationary dynamics.
\end{abstract}

%%%%%%%%%%%%%%%%%%%%%%%%%%%%%%%%%%%%%%%%%%%%%%%%%%%%%%%%%%%%%%%%%%%%%%%%%%%%%%%%
\section{INTRODUCTION}
\label{sec:intro}
\subsection{Motivation and Setting}
\label{ssec:motivation}
Learning-based control methods have shown strong performance in complex dynamical systems~\citep{hwangbo2019learning, 10610057, zakka2025mujoco, hafner2025mastering, zheng2025learning}. In particular, model-based reinforcement learning (MBRL), where a dynamics model is learned from data and used for planning, has emerged as a sample-efficient approach for learning in the real-world~\cite{hafner2025mastering, zheng2025learning, m2023model, hoffman2025learninglesssampleefficient}. Furthermore, MBRL methods enjoy strong theoretical guarantees for general nonlinear systems in the episodic~\cite{kakade2020information, curi2020efficient, sukhija_optimism_2025}, non-episodic~\cite{sukhija_optimism_2025}, and safe learning~\cite{as2024actsafe} settings. 
However, a common assumption across these algorithms is that the system dynamics remain constant during learning. This assumption is often violated in real-world systems due to drift, wear, changing operating conditions, or environmental variation. Examples include robotic systems with hardware degradation, vehicles under varying loads, or systems operating across different regimes.

In such settings, naively applying existing MBRL techniques results in suboptimal performance (\cref{sec:theory}). In particular, with time-varying dynamics,
data collected in the past may no longer be representative of the current system. As a result, reusing all historical data can lead to biased models and degraded control performance. This raises a fundamental question: \emph{how should RL algorithms adapt to time-varying dynamics?}

\pagebreak

We address this question and propose two MBRL  algorithms, R-OMBRL and SW-OMBRL, for sample-efficient learning in time-varying nonlinear systems with continuous state-action spaces.  Fundamental to both algorithms is the use of Bayesian models to learn an uncertainty-aware dynamics representation of the true system. 

Moreover, we use the epistemic uncertainty of our learned dynamics model as an intrinsic reward to direct exploration. Furthermore, we address the challenge of time-varying dynamics by carefully selecting the data used for training the agent. In particular, we consider periodic resets of the data buffer (R-OMBRL) and sliding windows (SW-OMBRL) that retain only recent data.
Our key contributions are
\begin{itemize}
    \item \textbf{Formulation:} We introduce an episodic model-based reinforcement learning framework for control under time-varying dynamics, together with a variation-budget model that captures temporal drift.
    
    \item \textbf{Theory:} We derive dynamic regret bounds for optimistic MBRL with restricted data buffers, showing how performance depends on the retained data horizon and the total variation in the dynamics.
    
    \item \textbf{Insight:} Our analysis reveals that limiting the influence of stale data is necessary to maintain calibrated uncertainty and achieve sublinear dynamic regret under non-stationarity.
    
    \item \textbf{Algorithms:} 
    Based on the theoretical insights, we
    propose practical modifications that enable the use of NNs for
    learning in high-dimensional systems. 
    
    \item \textbf{Experiments:} We demonstrate improved performance of the proposed methods over MBRL baselines on continuous control tasks with non-stationary dynamics.
\end{itemize}
%%%%%%%%%%%%%%%%%%%%%%%%%%%%%%%%%%%%%%%%%%%%%%%%%%%%%%%%%%%%%%%%%%%%%%%%%%%%%%%%
\section{RELATED WORK}

\subsection{Model-Based Reinforcement Learning}
MBRL algorithms are commonly applied for learning in the real-world~\cite{zheng2025learning, wu2023daydreamer, hansen2022modem}. Crucially, model-based RL algorithms learn a dynamics model of the underlying true system, and use the learned model for planning. However, to avoid overexploitation of an inaccurate model, \cite{curi2020efficient, kakade2020information} learn an uncertainty-aware model of the dynamics and use the uncertainty to enforce optimistic exploration. Recently, \cite{sukhija_optimism_2025} propose a scalable and efficient optimistic exploration strategy where the model epistemic uncertainty is used as an intrinsic reward to direct exploration. Under common regularity assumptions on the true system, they provide sublinear regret bounds for the algorithm in the finite-horizon, discounted infinite-horizon, and average reward RL settings. 

Nonetheless, crucial to their theoretical analysis is that the true system remains constant during learning. In this work, we go beyond this setting and tackle the more realistic problem of learning under non-stationarity.

\subsection{Learning under Non-Stationarity}
\label{ssec:time varying RL}

Learning in non-stationary environments has been extensively studied in various contexts.
The most developed theory exists in bandit and Bayesian optimization settings. Time-varying GP bandits capture non-stationarity via stochastic drift in function space~\cite{bogunovic_time-varying_2016, brunzema_event-triggered_2025}, or using variation budgets~\cite{zhou2019no}, leading to dynamic regret guarantees. However, these algorithms maximize and observe immediate reward and
do not tackle the more general RL setting, where high-dimensional policies are learned using trajectories acquired through rollouts on the true system.

In RL, non-stationarity has mainly been studied in tabular settings by controlling cumulative changes in rewards and transitions~\cite{ortner_variational_2020}, or even without prior knowledge of the variation in model-free settings \cite{wei_non-stationary_2021}. Notably, \cite{cheung2020reinforcement} propose a sliding-window based MBRL algorithm that builds on top of the seminal UCRL~\cite{auer2008near} algorithm for optimistic exploration. While similar to our approach in spirit, these methods tackle the finite state-action settings and do not consider the continuous case, which is ubiquitous in real-world systems.
%%%%%%%%%%%%%%%%%%%%%%%%%%%%%%%%%%%%%%%%%%%%%%%%%%%%%%%%%%%%%%%%%%%%%%%%%%%%%%%%
\section{PROBLEM FORMULATION}
\label{sec:problem}

\subsection{Control Objective}
\label{ssec:objective}
We consider finite-horizon control of a dynamical system with state space $\setX \subset \R^{d_x}$ and action space $\setU \subset \R^{d_u}$. 
The performance of a policy $\vpi:\setX\to\setU$ under dynamics $\vf$ is measured by the finite-horizon return
\begin{equation}
    J(\vpi, \vf) = \mathbb{E}^{\vpi, \vf}_{\vw_{0:T-1}} \left[\sum_{t=0}^{T-1} r(\vx_{t}, \vu_{t})\right],
    \label{eq:control_objective}
\end{equation}
where the expectation is taken over the 
Gaussian noise $\vw_{0:T-1} \overset{iid}{\sim} \mathcal{N}(0, \sigma^2 \mI)$, the superscript $\vpi$ denotes that the inputs are selected under the policy $\vu_{t} = \vpi(\vx_{t})$, and $\vf$ denotes that the state evolves according to $\vx_{t+1}=  \vf(\vx_t, \vu_t) + \vw_t$
starting from initial state $\vx_0$. The summand is given by the reward function $r: \setX \times \setU \to [0,R_{\text{max}}] \subset \R$ evaluated at the states and inputs along this trajectory. 

\subsection{Time-Varying Dynamics}
\label{ssec:dynamics}
We consider an unknown discrete-time, time-varying dynamical system. For an episode $n \in \{1, \dots, N\}$, the system dynamics are given by
\begin{equation}
    \vx_{n,t+1} = \vf_n^*(\vx_{n,t}, \vu_{n,t}) + \vw_{n,t}, \quad t = 0, \dots, T-1,
    \label{eq:dynamics}
\end{equation}
where $\vf_n^*$ denotes the (unknown) dynamics at episode $n$. 
We study the \emph{episodic} RL setting, where at the beginning of each episode the initial state is sampled from an initial state distribution $\rho$, i.e., $\vx_0 \sim \rho$. Then we rollout a policy $\vpi \in \Pi$ for $T$ steps in the environment and collect the resulting trajectory for learning. 
The goal of the RL agent is to find the policy that maximizes the cumulative sum of rewards, i.e., solve (\ref{eq:control_objective}) for the dynamics $\vf_n^*$. 
For simplicity of notation, consider the case where the dynamics  $\vf_n^*$ are fixed within an episode.\footnote{Our results can be extended to the setting where the dynamics also change during the episode.} 

At the beginning of each episode $n$, the agent selects a policy $\vpi_n$, which is executed for $T$ steps under $\vf_n^*$. The corresponding optimal policy is defined as $\vpi_n^* = \arg\max_{\vpi\in\Pi} J(\vpi, \vf_n^*)$.
Since the dynamics vary across episodes, the optimal policy also generally varies and depends on the current episode.

\subsection{Performance Metric}
\label{ssec:performance}
To evaluate performance in the presence of time-varying dynamics, we consider \emph{dynamic regret}, defined as
\begin{equation}
    R_N = \sum_{n=1}^{N} \big[J(\vpi_n^*, \vf_n^*) - J(\vpi_n, \vf_n^*)\big].
\end{equation}
Dynamic regret compares the learner against the sequence of episode-wise optimal policies $(\vpi_n^*)_{n\geq 1}$.
This is in contrast to \emph{static regret}, which compares against a single fixed policy and is inappropriate in non-stationary environments where the optimal policy may change over time.

\subsection{Non-Stationarity Model}
\label{ssec:variation budget}
To quantify temporal variation in the dynamics, we adopt a frequentist model based on reproducing kernel Hilbert spaces (RKHS). We assume that for all episodes $n$, each component of the dynamics $\vf_n^*$ lies in a common RKHS $\mathcal{H}_k$ with bounded norm: \[
\setH^{d_x}_{k, B} \defeq \{\vf \mid \normHspace{f_j} \leq B, j=1, \dots, d_x\}
\]
for some kernel satisfying $k(\vz, \vz) \leq \sigma_{\max}$ for all $\vz \in\setZ$.

To capture non-stationarity, we impose a \emph{variation budget} on the sequence of dynamics:
\begin{equation}
    \sum_{n=1}^{N-1} \|\vf_{n+1}^* - \vf_n^*\|_{\mathcal{H}_k^{d_x}} \le P_N,
\end{equation}
where $\|\cdot\|_{\mathcal{H}_k^{d_x}}$ denotes the $d_x$-dimensional RKHS norm. In particular, for $\vf \in \mathcal{H}_k^{d_x}$, let 
$$\|{\vf}\|_{\mathcal{H}_k^{d_x}}={\sum_{j=1}^{d_x} \|f_j\|_{\mathcal{H}_k}}.$$

This assumption allows the dynamics to evolve over time while controlling the total amount of variation. Generally, $P_N$ will increase with $N$ as the dynamics evolve.
%%%%%%%%%%%%%%%%%%%%%%%%%%%%%%%%%%%%%%%%%%%%%%%%%%%%%%%%%%%%%%%%%%%%%%%%%%%%%%%%
\section{GAUSSIAN PROCESS DYNAMICS MODELS}
\label{sec:gp_models}

\subsection{GP Model}

For our theoretical analysis, we model the unknown dynamics with Gaussian processes (GPs). GPs are a particularly promising class of models because they are expressive, Bayesian,  and also enable thorough theoretical analysis of the algorithm.
Furthermore, they have a closed-form solution for the posterior mean and epistemic uncertainty.

Let $\vz_{n,t} \defeq (\vx_{n,t},\vu_{n,t}) \in \setZ \defeq \setX \times \setU$. For each state dimension $j \in \{1,\dots,d_x\}$, we model the $j$-th component of the dynamics,
$f_{n,j}^* : \setZ \to \R$ using an independent GP with kernel $k$. We assume that this kernel satisfies $k(\vz, \vz) \leq \sigma_{\max}$ for all $\vz \in \setZ$.
Consider then the trajectory at episode $n$ as follows
\[
    \vtau_n
    \defeq
    \left\{
        \left(\vz_{n,t}, \vy_{n,t}\right)
        \;\middle|\;
        0 \le t < T-1
    \right\},
\]
where the state-action tuples and the observations,
\[
    \vz_{n,t} \defeq (\vx_{n,t}, \vpi_n(\vx_{n,t})) \mbox{ and }
    \vy_{n,t} \defeq \vx_{n,t+1},
\]
come from the rollout of $\vpi_n$ on $\vf^*_n$.
Fitting the GP to the data collected in episodes $m, m+1, \dots, \ell$ results in the posterior mean and covariance functions at arbitrary $\vz\in \setZ$ given by
\begin{align}
\mu_{m:\ell,j}(\vz)
&=
\vk_{m:\ell}(\vz)^\top (\mK_{m:\ell} + \sigma^2 \mI)^{-1} \mathbf{y}^j_{m:\ell},\label{eq:generic_gp}
\\
\sigma_{m:\ell,j}^2(\vz)
&=
k(\vz,\vz) - \vk_{m:\ell}(\vz)^\top (\mK_{m:\ell}+ \sigma^2 \mI)^{-1} \vk_{m:\ell}(\vz), 
\notag
\end{align}
where $\mathbf{y}_{m:\ell}^j$ denotes the vector of $j$-th components of
the observed next states over the episodes $m$ to $\ell$, while $\vk_{m:\ell}(\vz) = \bigl[k(\vz,\vz_{i,t})\bigr]_{i}$ and $\bm{K}_{m:\ell}=\bigl[k(\vz_{h,t},\vz_{i,t})\bigr]_{h, i}$ are the data kernel vector and matrix over the transitions $\vz_{h,t},\vz_{i,t}$ in the dataset recorded over episodes $m:\ell$, i.e., $\setD_{m:\ell}\defeq\bigcup_{s=m}^{\ell}\vtau_s$.

\subsection{Maximum Information Gain}
Crucial to our theoretical results is the {\em maximum information gain} of kernel $k$~\cite{srinivas_gaussian_2012},
\begin{equation}
    {\gamma}_{N}(k) = \max_{\setA \subset \setX \times \setU; |\setA| \leq NT}  \frac{1}{2}\log\det\left({\mI + \sigma^{-2} {\bm K}(\setA)}\right).
    \label{eq:max_info_gain}
\end{equation}
where ${\bm K}(\setA) = [k(\vz_{h,t}, \vz_{i,t})]_{h,i}$ for all pairs $\vz_{h,t}, \vz_{i,t} \in \setA$. 
The maximum information gain $\gamma_{N}$ is a measure of the complexity for learning $\vf^*$ from $N$ episodes and
is sublinear for many kernels.\footnote{e.g., $\setO(\log^{d_{x} + d_{u} +1}(NT))$ for the squared exponential (RBF) kernel, $\setO((d_{x} + d_{u})\log(NT))$ for the linear kernel; see \cite{srinivas_gaussian_2012} for more detail.}

\subsection{Stale Data Problem}

In stationary settings, it is natural to fit the GP using all previously observed data. In the present non-stationary case, doing so introduces a systematic mismatch: older samples were generated by past dynamics $\vf_s^*$, whereas control at episode $n$ depends on the current dynamics $\vf_n^*$.

If the model is fitted to all past data, then the posterior mean $\vmu_{1:n-1}$ over $n-1$ episodes is biased toward outdated dynamics. Importantly, the GP variance $\vsigma_{1:n-1}^2$ does not capture this temporal mismatch, since it reflects only statistical uncertainty under a stationary-function assumption. Thus, even if $\vsigma_{1:n-1}$ is small, the true prediction error
%\begin{equation*}
 $ \norm{\vf_n^*(\vz) - \vmu_{1:n-1}(\vz)}$
%\end{equation*}
may remain large due to drift.
This motivates explicitly restricting the retained data horizon.

\subsection{Forgetting Mechanisms}

We consider two mechanisms to control the influence of stale data.

\paragraph*{Full resets}
Fix a reset period $H \in \N$. Let
\[
n_0(n) \defeq H\left\lfloor \frac{n-1}{H} \right\rfloor + 1
\]
such that at $n_0(n) - 1$ the data buffer is emptied, i.e., ``stale'' data is removed from the buffer. Therefore,
at episode $n$, the model is fitted only on data collected since the last reset, i.e., the dataset $\setD_{n_0(n):n-1}$. 

\paragraph*{Sliding window}
Fix a window size $w \in \N$. At episode $n$, the model is fitted only on the most recent $w$ episodes, i.e. on the dataset $\setD_{n-w:n-1}$. 

\subsection{Calibration Under Drift}
In the stationary case, where the dynamics do not change across episodes, the true dynamics $\vf_j^*$ is fixed.
Standard GP concentration results~\cite{srinivas_gaussian_2012, chowdhury_kernelized_2017, rothfuss_hallucinated_2023} then imply that, there exists $\beta_n(\delta) = B + \sigma\sqrt{2(\gamma_{n}  + d_x \log(\sfrac{1}{\delta}))}$ such that with probability at least $1-\delta$
\begin{equation}
\label{eq:gp_calibration_stationary}
|f_j^*(\vz)-\mu_{1:n-1,j}(\vz)|
\le
\beta_n(\delta)\, \sigma_{1:n-1,j}(\vz)
\ \forall \vz \in \setZ.
\end{equation}
Here, $\beta_n(\delta)$ represents the width of the confidence interval around the mean within which the true system is captured.  Effectively, in this setting $\beta_n(\delta)\, \sigma_{1:n-1,j}(\vz)$ is a proxy for the error of our model estimate  $\mu_{1:n-1,j}(\vz)$. 

However, these results hold for the stationary setting. For time-varying dynamics, the temporal drift introduces an additional bias term that is not captured by the inequality above. In the following Lemma, we extend these confidence bounds to the non-stationary case.

\begin{lemma}[Lemmas 1 and 2 of \cite{zhou2019no} adapted for vector function $\vf^*$]
\label{lem:calibration}
Assume the dynamics satisfy the RKHS regularity assumptions of \cref{sec:problem}. 
    Consider the GP mean and variance estimates $(\mu_{m:n-1,j}, \sigma_{m:n-1,j})$ with either the full resetting mechanism $m=n_0(n)$ or the sliding window $m=n-w$. It holds with probability at least $1-\delta$ that for any episode $n\in\{1, \dots, N\}$, any $\vz \in \setZ$, and any $j=1, \dots, d_x$
    \begin{align}
    \label{eq:calibration_nonstationary}
        \big|f_{n,j}^*(\vz) - \mu_{m:n-1,j}&(\vz)\big|
        \le
        \beta_n(\delta,n-m)\,\sigma_{m:n-1,j}(\vz)
        \notag\\
        &+\xi_{n-m}\,\sum_{s=m}^{n-1}
        \norm{f_{s+1,j}^* - f_{s,j}^*}_{\mathcal H_k},
    \end{align}
    where the confidence parameter is given by
    \begin{align*}
        &\beta_n(\delta, n-m)= \\
        &\begin{cases}
B + \sigma\sqrt{2\big(\gamma_{n-m} + d_x \log(\frac{1}{\delta})\big)} & \text{if full reset}  \\
    B + \sigma\sqrt{2\big(\gamma_{n-m} + d_x \log(\frac{nT}{\delta})\big)} & \text{if sliding window}.
\end{cases}
    \end{align*}
    and the coefficient on the drift term is 
\begin{align*}\xi_{n-m} = \frac{2\sigma_{\max}\sqrt{(n-m)(1+\sigma^2) \gamma_{n-m}}}{\sigma^2}.
\end{align*}
\end{lemma}

Equation (\ref{eq:calibration_nonstationary}) reveals two distinct sources of error:
\begin{enumerate}
    \item the standard epistemic uncertainty term $\beta \sigma$, and
    \item an additive temporal bias term caused by the accumulated drift in the dynamics over the past $n-m$ episodes, i.e., since the last reset at episode $n_0(n)$ or within the sliding window of size $w$.
\end{enumerate}

\section{OPTIMISTIC MODEL-BASED RL UNDER NON-STATIONARY DYNAMICS}
\label{sec:algorithms}

The analysis in \cref{sec:gp_models} shows that, under time-varying dynamics, the calibrated uncertainty depends explicitly on the retained data horizon. Motivated by this insight, we adapt an optimistic model-based reinforcement learning framework \cite{sukhija_optimism_2025} to the non-stationary setting by restricting the data buffer used for model learning and policy optimization. To this end, we first demonstrate how the model error propagates through the control cost, and use this to propose an optimistic synthesis algorithm.

\subsection{Performance Difference Bound}

We now relate model error to control performance akin to \cite{kakade_approximately_2002}. To unify the reset and sliding-window cases, let $\vmu_{m:n-1}$ and $\vsigma_{m:n-1}$ denote an arbitrary dynamics model and uncertainty model constructed using (\ref{eq:generic_gp}) from the most recent retained data horizon from episodes $m:n-1$, with $m$ as defined in \cref{lem:calibration}.

For a policy $\vpi$, let $J(\vpi,\vf_n^*)$ and $J(\vpi,\vmu_{m:n-1})$ denote the expected finite-horizon returns from (\ref{eq:control_objective}) under the true dynamics and the learned model, respectively. We further define the trajectory-wise accumulated epistemic uncertainty for a range of episodes $m, \dots, \ell$ as
\begin{equation}
\label{eq:sigma_trajectory}
\Sigma_{m:\ell}(\vpi,\vf)
\defeq
\E^{\vpi, \vf}_{\vw_{0:T-1}}\!\left[
\sum_{t=0}^{T-1}
\norm{\vsigma_{m:\ell}(\vx_{n,t},\vpi(\vx_{n,t}))}_2
\right],
\end{equation}
where the expectation is taken over trajectories generated by policy $\vpi$ under the dynamics $\vf$ with the same process noise as in (\ref{eq:dynamics}).
The bound is then as follows.

\begin{lemma}[Performance difference bound]
\label{lem:performance_difference}
Assume the dynamics satisfy the regularity assumptions of \cref{sec:problem}. Let $(\mu_{m:n-1}, \sigma_{m:n-1})$ be the GP model fit with either resetting or forgetting. 
Conditioned on the success event of \Cref{lem:calibration}, it holds for every episode $n$ and every policy $\vpi$ that
\begin{align}
\big|
J(&\vpi,\vf_n^*) - J(\vpi,\vmu_{m:n-1})
\big| \le B_{n,m} \label{eq:informal_perf_diff}\\
&+
\lambda_{n,m} \, \min\{\Sigma_{m:n-1}(\vpi,\vf_n^*), \Sigma_{m:n-1}(\vpi,\vmu_{m:n-1})\}.\notag
\end{align}
where
\begin{equation}
\label{eq: exploration bonus weight}
\lambda_{n,m} \defeq \frac{R_{\max} T}{\sigma}\,\beta_n(\delta,n-m)
\end{equation}
scales the contribution of epistemic uncertainty, and
\begin{equation}
B_{n,m}
\defeq
\xi_{n-m}\frac{R_{\max} T^2}{\sigma} \sum_{s=m}^{n-1}
\norm{\vf_{s+1}^*-\vf_{s}^*}_{\mathcal H_k}
\label{eq: temporal drift accumulation}
\end{equation}
collects the temporal drift over the recorded $n-m$ episodes.
\end{lemma}

Thus, the policy performance difference decomposes into two terms: an uncertainty term, governed by the accumulated posterior standard deviation along the trajectory, and a temporal bias term, governed by the amount of drift contained in the retained dataset.
Most importantly, the term $B_{n, m}$ is independent of the policy. We use this key insight to integrate the optimistic MBRL framework in the non-stationary case.
\subsection{Optimistic MBRL Framework}

At each episode $n$, the algorithm maintains a policy $\vpi_n$, a dynamics model $\mathcal{M}_n = (\vmu_{m:n-1},\vsigma_{m:n-1})$, and a data buffer $\mathcal{D}_{m:n-1}$. Given the current model, the policy is updated by solving an optimistic planning problem under the learned mean dynamics:
\begin{equation}
    \vpi_n = \underset{\vpi}{\arg\max} \; J(\vpi, \vmu_{m:n-1})+ \lambda_{n,m}\Sigma_{m:n-1}(\vpi, \vmu_{m:n-1}). 
    \label{eq:optimistic_objective}
\end{equation}
Here, the uncertainty term $\Sigma_{m:n-1}(\vpi, \vmu_{m:n-1})$ acts as an intrinsic reward and directs exploration toward poorly modeled regions of the state-action space, and $\lambda_{n,m}$ from \eqref{eq: exploration bonus weight} is a positive constant which is used to trade off maximizing the reward and model uncertainty. 

In stationary settings, such optimism-based methods train both the policy and the dynamics model on all data collected so far. In the present non-stationary setting, however, using all previously collected data can induce substantial bias due to stale data. Our main algorithmic modification is therefore to restrict the data buffer to recent data only.

\subsection{Algorithms}

We consider two mechanisms for limiting stale data.

\paragraph*{R-OMBRL (reset-based, \cref{alg:reset_mbrl})}
The data buffer is reset every $H$ episodes. Hence, at episode $n$, both the dynamics model and the policy are trained only on data collected since the most recent reset.

\paragraph*{SW-OMBRL (sliding window, \cref{alg:windowed_mbrl})}
The data buffer retains only the most recent $w$ episodes, discarding older samples. Hence, at episode $n$, both the dynamics model and the policy are trained only on transitions from this window. 

In both cases, restricting the data buffer limits the temporal bias identified in (\ref{eq:calibration_nonstationary},\ref{eq:informal_perf_diff}).

We empirically validate our theoretical results for Gaussian Processes in \cref{fig:gps}. We compare R-OMBRL and SW-OMBRL against SOMBRL from \cite{sukhija_optimism_2025}, which is a MBRL algorithm that optimizes for the objective in \eqref{eq:optimistic_objective}, but without restricting the data buffer to train the GP for the model $(\vmu_{1:n-1},\vsigma_{1:n-1})$. We use the Pendulum environment from Gym \cite{brockman_gym_2016}, adapted to a non-stationary setting where at one point, the maximum applicable action $\vu_t$ decays rapidly.

In this application with GP dynamics, we find that R-OMBRL and SW-OMBRL both outperform SOMBRL: They recover performance and adapt to the changed dynamics, whereas the baseline model $(\vmu_{1:n-1}, \vsigma_{1:n-1})$, trained on all past data, remains biased toward outdated dynamics and fails to adapt.

\begin{algorithm}
\caption{\emph{R-OMBRL:} Reset-based optimistic MBRL}
\label{alg:reset_mbrl}
\begin{algorithmic}[1]
\Require Horizon $T$, reset period $H$ 
\State Initialize $\vpi_\theta$, $\mathcal{M}_{0}$, empty data buffer $\mathcal{D}=\emptyset$
\For{episode $n = 1,2,\dots,N$}
    \If{$n \bmod H = 0$}
        \State $\mathcal{D} \gets \emptyset$
    \EndIf
    \State Collect trajectory $\bm{\vtau}_n$ using $\vpi_{\bm{\theta}}$
    \State $\mathcal{D} \gets \mathcal{D} \cup \bm{\vtau}_n$
    \State Update $\mathcal{M}_{\bm{\phi}}$ using $\mathcal{D}$
    \State Update $\vpi_{\bm{\theta}}$ by \eqref{eq:optimistic_objective} using $(\mathcal{M}_{\bm{\phi}}, \setD)$
\EndFor
\end{algorithmic}
\end{algorithm}

\begin{algorithm}
\caption{\emph{SW-OMBRL:} Sliding-window optimistic MBRL}
\label{alg:windowed_mbrl}
\begin{algorithmic}[1]
\Require Horizon $T$, window size $w$
\State Initialize $\vpi_{\bm{\theta}}$, $\mathcal{M}_{0}$, empty data buffer $\mathcal{D}=\emptyset$
\For{episode $n = 1,2,\dots,N$}
    \State Collect trajectory $\bm{\vtau}_n$ using $\vpi_{\bm{\theta}}$
    \State $\mathcal{D} \gets \mathcal{D} \cup \bm{\vtau}_n$
    \State Remove data older than $w$ episodes from $\mathcal{D}$
    \State Update $\mathcal{M}_{\bm{\phi}}$ using $\mathcal{D}$
    \State  Update $\vpi_{\bm{\theta}}$ by \eqref{eq:optimistic_objective}  using $(\mathcal{M}_{\bm{\phi}}, \setD)$
\EndFor
\end{algorithmic}
\end{algorithm}

\begin{figure}
    \centering
    \includegraphics[width=.9\linewidth]{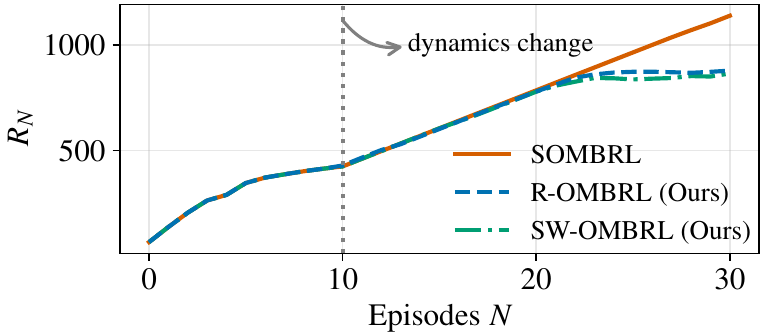}\vspace{-3mm}
    \caption{Learning curves for the setting with GP dynamics on the Pendulum environment. We report the mean cumulative regret $R_N$ over 5 random seeds. At episode $N=10$, we induce a change in dynamics by limiting the maximum applicable action $\vu_{N,t}$ to half its original value over time. This leads to a linear regret for the stationary SOMBRL baseline, while R-OMBRL and SW-OMBRL adapt to the change in dynamics.}
    \label{fig:gps}
\end{figure}

\section{THEORETICAL ANALYSIS}
\label{sec:theory}

\subsection{Main Regret Bound}
Combining the calibration result with the performance difference bound yields a dynamic regret bound. 
We state an informal version here and defer the full theorem and proof to 
\ifpreprint
Appendix \ref{app:proofs}.
\else
the extended manuscript.
\fi

\begin{theorem}[Regret bound]
\label{th:main_regret}
Assume the dynamics satisfy the RKHS regularity and variation-budget assumptions of \cref{sec:problem}. Then, the optimistic MBRL algorithm of \eqref{eq:optimistic_objective} with either resetting or a sliding window achieves dynamic regret that satisfies
\begin{equation}
R_N
=
\tilde{\mathcal O}\!\left(
N \sqrt{\frac{\gamma_{p}^3}{p}} +
\gamma_{p} p^{3/2} P_N
\right),
\end{equation}
where $p = H$ for resetting and $p=w$ for the sliding window, and $\tilde{\mathcal O}$ hides the dependence of the episode length and factors that are logarithmic in $N$.
\end{theorem}

\subsection{Interpretation for R-OMBRL}
In the following, we interpret the regret bound for the R-OMBRL algorithm. The analysis is analogous for SW-OMBRL.

The regret bound of \Cref{th:main_regret} decomposes into a learning term and a drift term. For the reset case (R-OMBRL),
$N \sqrt{{\gamma_{H}^3}/{H}}$
decreases with larger $H$, reflecting improved performance from reusing more data, whereas
$\gamma_{H}\, H^{3/2} P_N$
increases with $H$, reflecting the growing influence of stale data if the reset period increases.
Hence $H$ controls a bias-variance-type trade-off:
\begin{itemize}
    \item small $H$: strong adaptivity to changing dynamics, but less data per model fit,
    \item large $H$: more data for training the model but with a larger bias.
\end{itemize}
In the stationary case, OMBRL~\cite{sukhija_optimism_2025} -- where $H = N$ -- has a regret of 
$$
R_N
=
\mathcal O\!\left(
\sqrt{N \gamma_{N}^3} +
\gamma_{N} N^{3/2} P_N
\right),
$$
which is only sublinear if $P_N = 0$, i.e., no variation in the dynamics takes place.
In our case, we can achieve sublinear regret for $P_N > 0$ by carefully selecting $H$. 
For instance, similar to \cite{zhou2019no}, by selecting $H \propto {\gamma^{1/4}_N N^{1/2}}$, we get
$$
R_N
=
\mathcal O\!\left(
\gamma^{\sfrac{11}{8}}_{N}(1 + P_N) N^{3/4}
\right).
$$
For a specific total variation rate, e.g., $P_N \propto \log(N)$, this results in a regret $
R_N
=
\mathcal O\!\left(
\gamma^{\sfrac{11}{8}}_{N} N^{3/4} \log(N)
\right),
$
which is sublinear for common kernels such as the exponential and linear
kernel~\cite{srinivas_gaussian_2012}. Moreover, while the current algorithm requires $H$ to be set \emph{a priori} based on the duration $N$, for specific choices of $\gamma_N$ and $P_N$ (e.g., the exponential kernel example above), we can convert our regret to an anytime regret bound by applying the doubling trick~\cite{besson2018doubling}.

\subsection{Practical Modification}
Our theoretical analysis focuses on GP models, where we can guarantee calibration of our learned model and bound the complexity of learning the dynamics, i.e., $\gamma_n$. However, GPs scale poorly to high-dimensional systems. Furthermore, they are also computationally expensive.\footnote{Cubic in the number of data points.} Accordingly, most MBRL algorithms use neural networks for representing the dynamics. In particular, works such as \cite{chua2018deep,
curi2020efficient} use 
Bayesian neural networks (BNNs), specifically deep ensembles~\cite{lakshminarayanan_ensembles_2017},  to learn an uncertainty-aware dynamics model. 

To this end, we propose practical modifications to R-OMBRL and SW-OMBRL, which enable learning with BNNs. Moreover, we employ a scalable parametrized uncertainty-aware dynamics model $\mathcal{M}_{\bm{\phi}} = (\vmu_{\bm{\phi}},\vsigma_{\bm{\phi}})$ and instantiate it using deep ensembles. The model is trained to maximize the data likelihood using stochastic gradient descent.
Similar to \cite{sukhija_optimism_2025}, we train a parameterized policy $\vpi_{\bm{\theta}}$ using the MBPO algorithm~\cite{janner2019trust} on the optimistic objective (\ref{eq:optimistic_objective}).  
To handle non-stationarity, we restrict the data buffer through resets or sliding windows as described in \cref{sec:algorithms}. 

In contrast to the theoretically derived exploration weight $\lambda_{n,m}$ from \eqref{eq: exploration bonus weight}, which depends explicitly on confidence bounds and the retained data horizon, it can be beneficial in practice to treat this coefficient as a tunable parameter. For instance, \cite{sukhija2025maxinforl} propose an automatic tuning mechanism for this exploration weight and adapt it online. Moreover, \cite{iten2026sampleefficient} study this trade-off empirically in the model-based setting and demonstrate that adaptive or manually tuned exploration weights can perform well in practice.

Furthermore, we additionally employ soft resets of both model and policy parameters. In particular, at predefined intervals, we update using randomly sampled new parameters $(\bm{\theta}_0, \bm{\phi}_0)$ by
\begin{equation}
\bm{\phi} \leftarrow (1-\alpha_1)\bm{\phi} + \alpha_1 \bm{\phi}_0,
\quad
\bm{\theta} \leftarrow (1-\alpha_2)\bm{\theta} + \alpha_2 \bm{\theta}_0.%,
\label{eq:soft_reset}
\end{equation}
We ablate the effects of $(\alpha_1, \alpha_2)$ in
\ifpreprint
Appendix \ref{app:alpha}.
\else
the extended manuscript.\footnotemark{}
\fi
%%%%%%%%%%%%%%%%%%%%%%%%%%%%%%%%%%%%%%%%%%%%%%%%%%%%%%%%%%%%%%%%%%%%%%%%%%%%%%%%
\section{Experiments}
\label{sec:experiments}
\subsection{Setup}
\label{ssec:experiment_Setup}
We evaluate the proposed methods on continuous control benchmarks from Gym~\cite{brockman_gym_2016} and MuJoCo~\cite{todorov_mujoco_2012}.
To study learning under changing dynamics, we introduce controlled non-stationarity via parameter drift. 
 
We evaluate R-OMBRL and SW-OMBRL, which restrict the data buffer through resets and sliding windows, respectively. As a baseline, we use a stationary optimistic model-based RL method (OMBRL)~\cite{sukhija_optimism_2025}, a SOTA MBRL method for stationary dynamics, which trains on all collected data without any forgetting mechanism.
Additional implementation details and ablations are provided in 
\ifpreprint
Appendix \ref{app:experiments}.
\else
the extended manuscript.
\fi

\subsection{Non-Stationary Environments}
\label{ssec:decay_rates}
We evaluate on the Pendulum, HalfCheetah~\citep{wawrzynski_halfcheetah_2009}, and Hopper~\citep{erez_infinite_2012} environments by modifying the implementation from \cite{brockman_gym_2016} by introducing time-varying dynamics via an episode-dependent decay of actuator strength. 
Specifically, we scale the maximum admissible control as
\begin{equation}
    \bar{u}_{n} = \exp(-a n)\,(u_{\max}-u_{\min}) + u_{\min},
    \label{eq:max_torque}
\end{equation}
where $a \geq 0$ controls the rate of change.
Here, we treat the reset period $H$ and window size $w$ as tunable hyperparameters and tune the internal reward weight $\lambda_{n,m}$ using the method from \cite{sukhija2025maxinforl}. For the policy and model parameters $\vtheta$ and $\bm{\phi}$, we employ soft resets as described in \eqref{eq:soft_reset} with $\alpha_1=\alpha_2=0.2$.

\cref{fig:decay_rates} shows results across environments for three decay rates. 
We report dynamic cumulative regret with respect to an estimate of the optimal policy, obtained by running an RL algorithm (SAC~\cite{haarnoja2018soft}), independently for fixed actuator strengths until convergence.

\looseness=-1
Across all settings, the stationary baseline fails to track the evolving dynamics, leading to rapidly increasing regret. 
In contrast, both R-OMBRL and SW-OMBRL significantly reduce regret accumulation by limiting the influence of stale data. 
Overall, we conclude that restricting the data buffer improves robustness under non-stationarity and leads to consistently better tracking of the underlying time-varying system.

\begin{figure}
    \centering
    \includegraphics[width=\linewidth]{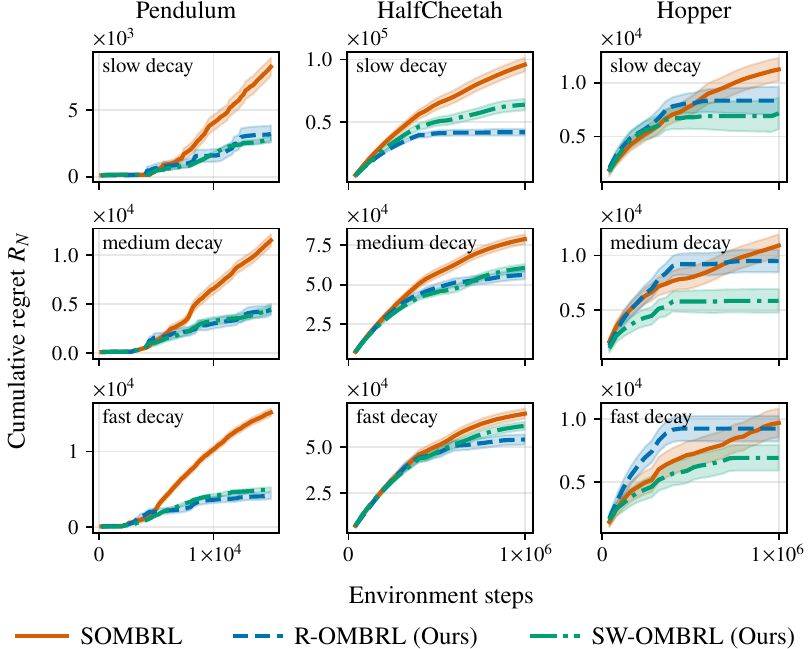}
    \vspace{-5mm}\caption{Dynamic regret under time-varying dynamics for different decay rates. 
    The stationary baseline SOMBRL accumulates large regret due to stale data, while R-OMBRL and SW-OMBRL improve tracking by restricting the data buffer. 
    We report the mean regret compared to an estimate of the optimal performance over five seeds with standard error.}
    \label{fig:decay_rates}
\end{figure}

\subsection{Different Environments}

We evaluate the methods across multiple MuJoCo environments~\cite{todorov_mujoco_2012}, comparing the stationary baseline (SOMBRL) with R-OMBRL and SW-OMBRL. 
The setup follows the previous experiment.

\cref{fig:different_envs} shows training under initially stationary dynamics, followed by a transition to time-varying dynamics induced by the decay in~\eqref{eq:max_torque}. 
The top row illustrates the evolution of the environment parameter, while the bottom rows report dynamic regret averaged over five seeds.

Under stationary dynamics, all methods perform similarly. 
After the onset of non-stationarity, the baseline accumulates substantially higher regret, while both R-OMBRL and SW-OMBRL adapt effectively by restricting the data buffer. 

These results demonstrate that data buffer adaptation improves performance across environments and scales to higher-dimensional control tasks.

\begin{figure*}
    \centering
    \includegraphics[width=\linewidth]{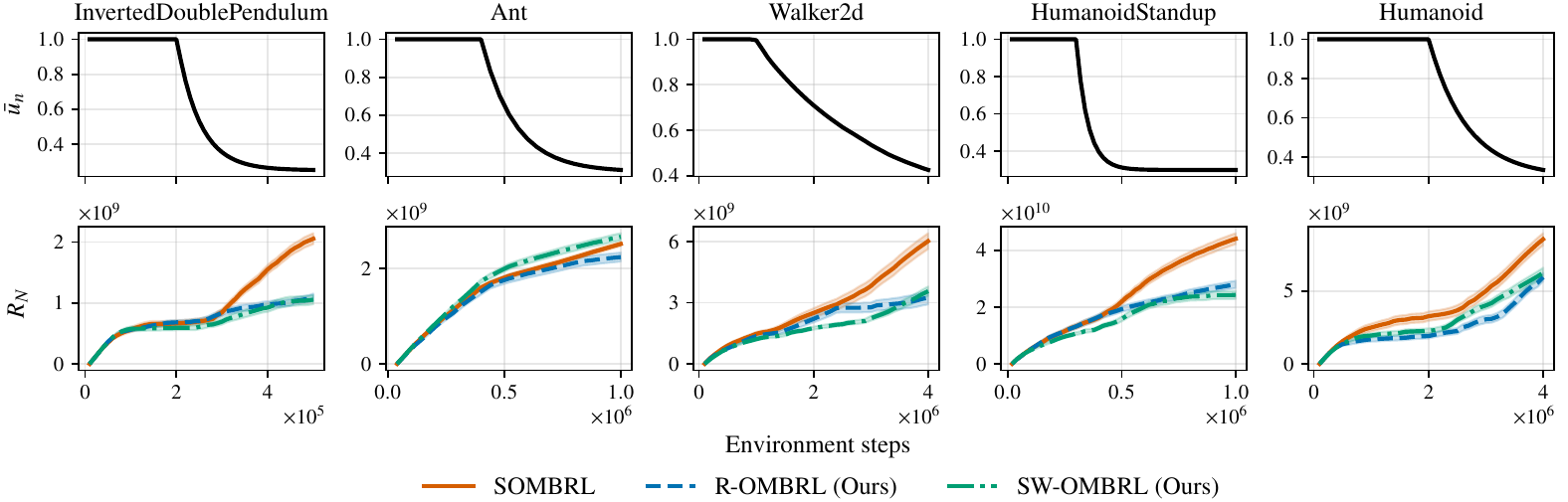}\vspace{-6.5mm}
    \caption{Dynamic regret across multiple environments. The top row shows the evolution of the maximum admissible torque $\bar{u}_n$ and its decay over environment training steps, while the bottom rows report the cumulative regret $R_n$ averaged over five seeds with standard error. Under stationary dynamics, all methods perform similarly. After the onset of non-stationarity, R-OMBRL and SW-OMBRL significantly reduce regret compared to the stationary baseline.}
    \label{fig:different_envs}
\end{figure*}

\subsection{Hardware Experiments}

We evaluate the proposed methods on a real-world RC car platform following the setup of~\cite{rothfuss2024bridging}. 
The system consists of a high-torque racecar capable of highly dynamic maneuvers, including drifting, with the state capturing position, orientation, and velocities, and control inputs given by steering and throttle.

The task is a dynamic parking maneuver, where the car must rotate and park at a target location. 
At high actuator strength, the optimal behavior involves aggressive sliding and drifting. 
To induce non-stationarity, we introduce an episode-dependent decay of the maximum throttle, gradually transforming the task from a drift-based maneuver into a standard parking problem.

We compare R-OMBRL against the stationary baseline (SOMBRL) in \cref{fig:placeholder}. 
The top row shows rollouts of the learned policies after 30 episodes. While SOMBRL fails to adapt and is unable to complete the parking maneuver, R-OMBRL successfully performs the task. 
The bottom row reports the return on the real system over episodes. 
Before the onset of non-stationarity, both methods perform similarly, whereas after the throttle decay, R-OMBRL adapts more effectively and achieves higher returns.

These results mirror the simulation findings: restricting the data buffer improves adaptation to evolving dynamics. 
The stationary baseline accumulates outdated transitions generated under earlier, high-throttle regimes, leading to degraded performance. 
In contrast, R-OMBRL discards stale data and adapts to the current system, resulting in more stable and consistent behavior.

\begin{figure}
    \centering
    \includegraphics[width=\linewidth]{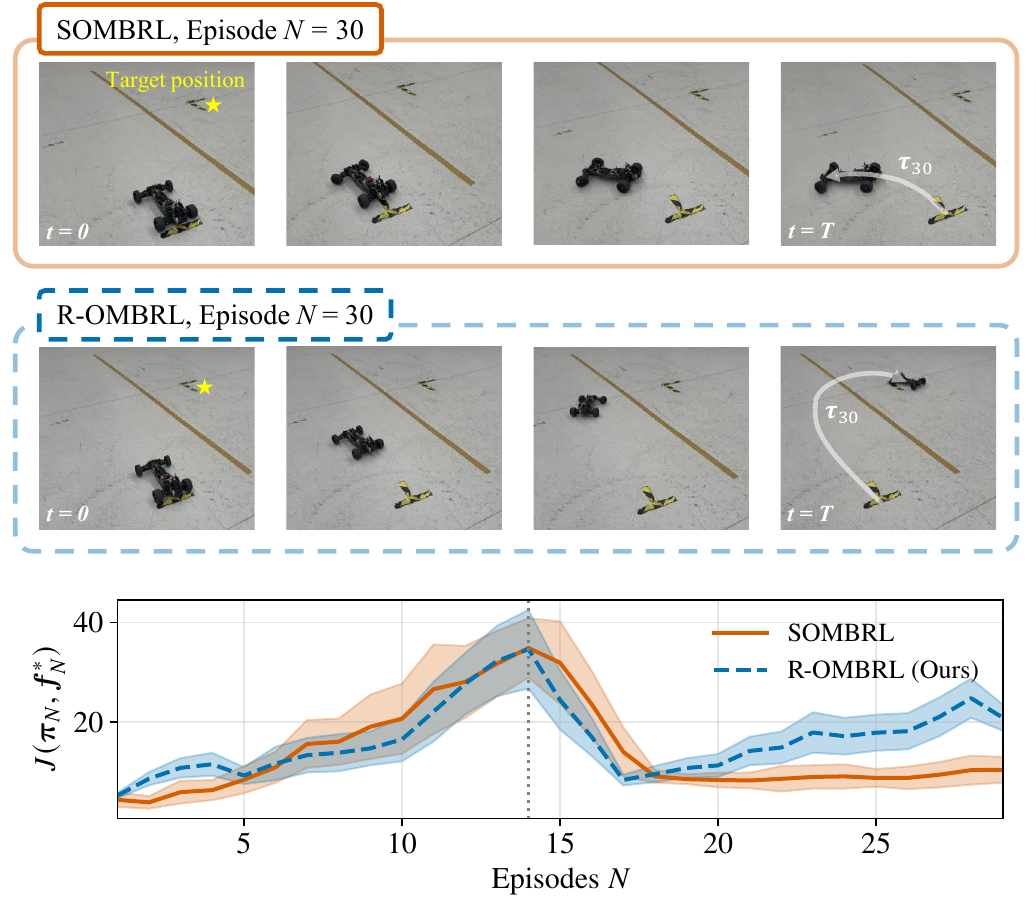}
    \vspace{-5mm}\caption{\looseness -1 Hardware experiments on a real RC car. The task is a parking maneuver that transitions from drift-based behavior to standard parking as the maximum throttle decays at $N=14$ episodes. The top row shows rollouts after $N=30$, where R-OMBRL successfully completes the task while SOMBRL fails to adapt. The bottom row shows the mean return and error bands along trajectories $\vtau_N$ on the real system over episodes, averaged over 6 random seeds. Restricting the data buffer enables adaptation to changing dynamics and improves performance compared to the stationary baseline.}
    \label{fig:placeholder}
\end{figure}

%%%%%%%%%%%%%%%%%%%%%%%%%%%%%%%%%%%%%%%%%%%%%%%%%%%%%%%%%%%%%%%%%%%%%%%%%%%%%%%%
\section{CONCLUSIONS AND FUTURE WORKS}

\subsection{Conclusions}

In this work, we study model-based reinforcement learning under time-varying dynamics. 
Our analysis shows that, under a variation-budget model of the dynamics, persistent non-stationarity requires explicitly limiting the influence of stale data to maintain calibrated uncertainty and achieve meaningful dynamic regret guarantees. 
Motivated by this insight, we propose practical optimistic MBRL algorithms, R-OMBRL and SW-OMBRL, based on data buffer resets and sliding windows. 
While similar ideas have been proposed by prior work, such as \cite{cheung2020reinforcement} for finite state-action spaces, to the best of our knowledge, we are the first to provide dynamic regret guarantees for the general setting of continuous spaces and non-linear dynamics. 
Furthermore, we also empirically validate the proposed methods on high-dimensional continuous control benchmarks as well as on real-world hardware. We show both in simulation and the real-world that R-OMBRL and SW-OMBRL significantly outperform the stationary MBRL baseline.  
This shows that data buffer adaptation is essential, in both theory and practice, for robust learning and control in a non-stationary system. 

\subsection{Future Works}

An important direction concerns selecting the reset period or window size adaptively based on the observed data could further improve performance, as \cite{brunzema_event-triggered_2025} do in the Bayesian optimization setting. In our experimental evaluation, we focus on parameter decay, which captures gradual performance degradation in many practical scenarios and satisfies the variation-budget assumptions. Our methods could be applied more broadly, evaluating them under alternative patterns such as abrupt shifts or cyclic variations.
Finally, extending the framework to non-episodic settings with continuously evolving dynamics remains an important open problem.

%%%%%%%%%%%%%%%%%%%%%%%%%%%%%%%%%%%%%%%%%%%%%%%%%%%%%%%%%%%%%%%%%%%%%%%%%%%%%%%%
%%%%%%%%%%%%%%%%%%%%%%%%%%%%%%%%%%%%%%%%%%%%%%%%%%%%%%%%%%%%%%%%%%%%%%%%%%%%%%%%
\section{ACKNOWLEDGEMENTS}

\looseness -1 B.~Lee was supported by a postdoctoral fellowship and C.~Li by a doctoral fellowship from ETH AI Center. B.~Sukhija was supported by ELSA (European Lighthouse on Secure and Safe AI) funded by the European Union under grant agreement No. 101070617. This project has
received funding from the Swiss National Science Foundation under NCCR Automation, grant agreement 51NF40 180545. Numerical simulations were performed on the ETH Zürich Euler cluster. Parts of this text were revised with the assistance of a large language model to aid or polish writing and to improve grammar and clarity; the authors remain responsible for all content.

%%%%%%%%%%%%%%%%%%%%%%%%%%%%%%%%%%%%%%%%%%%%%%%%%%%%%%%%%%%%%%%%%%%%%%%%%%%%%%%%

\putbib

%%%%%%%%%%%%%%%%%%%%%%%%%%%%%%%%%%%%%%%%%%%%%%%%%%%%%%%%%%%%%%%%%%%%%%%%%%%%%%%%
%%%%%%%%%%%%%%%%%%%%%%%%%%%%%%%%%%%%%%%%%%%%%%%%%%%%%%%%%%%%%%%%%%%%%%%%%%%%%%%%

\ifpreprint
\pagebreak
\onecolumn

% Optional: title for the appendix TOC
\begin{center}
    {\LARGE \textbf{Model-Based Reinforcement Learning for Control\\
    under Time-Varying Dynamics}}\\[0.8cm]
    
    {\Large Supplementary Material}\\[0.8cm]
    
    {\large
    Klemens Iten, Bruce Lee, Chenhao Li, Lenart Treven, Andreas Krause, Bhavya Sukhija}
\end{center}

\vspace{1.2cm}

\addcontentsline{toc}{section}{Appendix}

\startcontents[sections]

\section{Proofs}
\label{app:proofs}

\subsection{Assumptions}

In the following, we restate the assumptions from \cref{sec:problem} formally.
\begin{assumption}[Continuity of $\vf^*_n$ and $\vpi_n$]
\label{ass:continuity}
The dynamics model $\vf^*_n$ and all policies $\vpi \in \Pi$ are continuous. This is a fairly common assumption in nonlinear control~\cite{khalil_nonlinear_2015}.
\end{assumption}

\begin{assumption}[Bounded reward]
    \label{ass:reward}
    The scalar reward is bounded, i.e.~$r:\setX\times\setU\rightarrow [0,R_\textrm{max}]$.
\end{assumption}

\begin{assumption}[Process noise distribution]
The process noise is i.i.d.~Gaussian with variance $\sigma^2$, i.e., $\vw_t \stackrel{\textrm{i.i.d}}{\sim} \setN(\vzero, \sigma^2\mI)$.
\label{ass:noise_properties}
\end{assumption}

\begin{assumption}[Episodic stationarity of dynamics]
\label{ass:episodic_frequentist}
There is a sequence of dynamics $(\vf_n^*)_{n\ge1}$ with
$\vf_n^*:\R^{d_x}\times\R^{d_u}\rightarrow\R^{d_x}$ such that within each episode $n$,
\[
    \vx_{n,t+1} = \vf_n^*(\vx_{n,t}, \vu_{n,t}) + \vw_{n,t}, \qquad t = 0, \dots, T-1,
\]
and $\vf_n^*$ may change only across episode boundaries (i.e., does not vary with $t$ inside an episode).
\end{assumption}

\begin{assumption}[Reproducing kernel Hilbert space]
\label{ass:rkhs}
We assume that the functions $f^*_{n,j}$, $j \in \set{1, \ldots, d_x}$ lie in a RKHS with kernel $k$ and have a bounded norm $B$, that is 
\[
\vf_n^* \in \Hspace, \qquad
\textrm{with }\ \setH^{d_x}_{k, B} = \{\vf \mid \normHspace{f_j} \leq B, j=1, \dots, d_x\}.
\]
Moreover, we assume that $k(\vz, \vz) \leq \sigma_{\max}$ for all $\vx \in \setX$.
\end{assumption}

\begin{assumption}[Bounded temporal variation budget]
\label{ass:variation_budget}
For each coordinate $j \in \set{1, \ldots, d_x}$, the sequence $(f_{n,j}^*)_{n\geq 1}$ satisfies
\[
    \sum_{n=1}^{N-1}\normHspace{f_{n+1,j}^* - f_{n,j}^*} \le P_{N,j}.
\]
Equivalently, in vector-valued form,
\[
    \sum_{n=1}^{N-1}\norm{\vf_{n+1}^*-\vf_n^*}_{\mathcal H_k^{d_x}}
    \le P_N,
\qquad
\text{where}
\qquad
    \norm{\vf}_{\mathcal H_{k,B}^{d_x}}
    \defeq \sum_{j=1}^{d_x}\normHspace{f_j},
    \qquad
    P_N \defeq \sum_{j=1}^{d_x}P_{N,j}.
\]
\end{assumption}

\subsection{Proofs}

\begin{proof}[Proof of \cref{lem:performance_difference}]

We give the proof for $|J(\vpi_n, \vf_n^*) - J(\vpi_n, \vmu_{m:n-1})| \leq \lambda_{n,m} \Sigma_{m:n-1}(\vpi_n, \vmu_{m:n-1}) +B_{n,m}$. The same argument holds for the second inequality.  
    Let $J_{t+1}(\vpi, \vf, \vx)$ denote the cost-to-go from state $\vx$ at step $t+1$ onwards under the dynamics $\vf$ and policy $\vpi$:
\begin{equation*}
    J_{t+1}(\vpi, \vf, \vx) \defeq \E_{\vw_{0:T-1}}^{\vpi,\vf}\!\left[ \sum_{k=t+1}^{T-1} r(\vx_k, \vpi(\vx_k)) \right].
\end{equation*}
    
    Following the policy difference Lemma~\cite{kakade_approximately_2002,sukhija2024optimistic}: 
    \begin{align*}
        &J(\vpi_n, \vmu_{m:n-1}) - J(\vpi_n, \vf_n^*)
        \\
        &=  \E_{\vw_{0:T-1}}^{\vpi_n,\vmu_{m:n-1}}\left[\sum^{T-1}_{t=0} J_{t+1}(\vpi_n, \vf_n^*, \vmu_{m:n-1}(\vx_{t}, \vpi_n(\vx_t))+\vw_t) -  J_{t+1}(\vpi_n, \vf_n^*, \vf_n^*(\hallucinated{\vx}_{t}, \vpi_n(\hallucinated{\vx}_t)) + \vw_t)\right]. \\
    \end{align*}
    Therefore,
    \begin{align*}
        |J(\vpi_n, & \vmu_{m:n-1}) - J(\vpi_n, \vf_n^*)|\\
        &=  \left| \E_{\vw_{0:T-1}}^{\vpi_n, \vmu_{m:n-1}}\left[\sum^{T-1}_{t=0} J_{t+1}(\vpi_n, \vf_n^*, \vmu_{m:n-1}(\vx_{t}, \vpi_n(\vx_t))+\vw_t) -  J_{t+1}(\vpi_n, \vf_n^*, \vf_n^*(\hallucinated{\vx}_{t}, \vpi_n(\hallucinated{\vx}_t)) + \vw_t)\right]\right| \\
 &\leq \sum^{T-1}_{t=0}  \E^{\vpi_n, \vmu_{m:n-1}}_{\vw_{0:t-1}}\left[\left|\E_{\vw_{t}}\left[J_{t+1}(\vpi_n, \vf_n^*, \vmu_{m:n-1}(\hallucinated{\vx}_{t}, \vpi_n(\hallucinated{\vx}_t)) + \vw_t) -  J_{t+1}(\vpi_n, \vf_n^*, \vf_n^*(\hallucinated{\vx}_{t}, \vpi_n(\hallucinated{\vx}_t)) + \vw_t) \vert \vx_t \right]\right|\right].\\
    \end{align*}
Next, we bound the last term in the outer expectation:
\begin{align*}
    \big|\E_{\vw_{t}}\big[J_{t+1}(&\vpi_n, \vf_n^*, \vmu_{m:n-1}(\hallucinated{\vx}_{t}, \vpi_n(\hallucinated{\vx}_t)) + \vw_t) -  J_{t+1}(\vpi_n, \vf_n^*, \vf_n^*(\hallucinated{\vx}_{t}, \vpi_n(\hallucinated{\vx}_t)) + \vw_t) \vert \vx_t \big]\big| \\
    \leq& \sqrt{\max\left\{\E_{\vw_{t}}[J^2_{t+1}(\vpi_n, \vf_n^*, \vmu_{m:n-1}(\vx_{t}, \vpi_n(\vx_t))+\vw_t) \vert \vx_t],  \E_{\vw_{t}}[J^2_{t+1}(\vpi_n, \vf_n^*, \vf_n^*(\vx_{t}, \vpi(\vx_t))+\vw_t) \vert \vx_t]\right\}} \\
    &\times \min\left\{\frac{\norm{\vf_n^*(\vx_t, \vpi_n(\vx_t)) - \vmu_{m:n-1}(\vx_t, \vpi_n(\vx_t))}_2}{\sigma}, 1\right\} \tag*{(Lemma C.2 of \cite{kakade2020information})}\\
    \leq&\ \frac{R_{\max} T}{\sigma} \times \left({\beta_n(\delta, n-m)}\,\|\vsigma_{n}(\vx_t, \vpi_n(\vx_t))\|_2 +  \xi_{n-m} \sum_{j=1}^{d_x} \sum_{s=m}^{n-1}\|f^*_{s+1,j}-f^*_{s,j}\|_{\mathcal H_k} \right). \tag{by \cref{ass:reward} and the success event of \cref{lem:calibration}}
\end{align*}

Therefore, we have
\begin{align*}
    &|J(\vpi_n, \vmu_{m:n-1}) - J(\vpi_n, \vf_n^*)| \\
    &\le  \sum_{t=0}^{T-1} 
    \E_{\vw_{0:t-1}}^{\vpi_n,\vmu_{m:n-1}}\!\left[ \frac{R_{\max} T}{\sigma} \left({\beta_n(\delta, n-m)}\,\bigl\|\vsigma_{n}(\vx_t, \vpi_n(\vx_t))\bigr\|_2 + \xi_{n-m} \sum_{j=1}^{d_x} \sum_{s=m}^{n-1}\|f^*_{s+1,j}-f^*_{s,j}\|_{\mathcal H_k} \right) 
\right] \\
    &= \lambda_{n,m}\cdot \Sigma_{m:n-1}(\vpi_n, \vmu_{m:n-1}) + B_{n,m}.
\end{align*}
\end{proof}

\begin{lemma}[Episodic regret bound]
\label{lem:tv_episodic_regret_bound}
    Assume the dynamics satisfy the regularity assumptions of \cref{sec:problem}, i.e., let \cref{ass:continuity} to \ref{ass:variation_budget} hold. Let $(\mu_{m:n-1}, \sigma_{m:n-1})$ be the GP model fit with either resetting or forgetting. 
    Define the quantities
\begin{equation*}
\tilde\lambda_{n,m} \defeq \frac{\max\{R_{\max},\sigma_{\max}\} T}{\sigma}\,\beta_n(\delta,n-m) \quad\mbox{ and }\quad
\tilde B_{n,m}
\defeq
\xi_{n-m}\frac{\max\{R_{\max},\sigma_{\max}\} T^2}{\sigma} \sum_{s=m}^{n-1}
\norm{\vf_{s+1}^*-\vf_{s}^*}_{\mathcal H_k}.
\end{equation*}
    
     Conditioned on the success event of \Cref{lem:calibration}, it holds for every episode $n$ that
\begin{equation*}
     J(\vpi_n^*, \vf_n^*) - J(\vpi_n, \vf_n^*) \leq (\tilde\lambda^2_{n,m} + 2\tilde\lambda_{n,m}) \Sigma_{m:n-1}(\vpi_n, \vf^*_n) + (\tilde\lambda_{n,m} + 2)\; \tilde B_{n,m}.
\end{equation*}
\begin{proof}
We follow Lemma A.2 of \cite{sukhija_optimism_2025} with the additive term from \cref{lem:calibration}.
\begin{align}
J(\vpi_n^*, \vf_n^*) -& J(\vpi_n, \vf_n^*) \notag \\
    &\le J(\vpi^*_n, {\vmu}_{m:n-1}) + \lambda_{n,m} \Sigma_{m:n-1}(\vpi^*_n, {\vmu}_{m:n-1}) +  B_{n,m} - J(\vpi_n, \vf^*_n) \tag{by \cref{lem:performance_difference}} \\
    &\leq J(\vpi_n, {\vmu}_{m:n-1}) + \lambda_{n,m} \Sigma_{m:n-1}(\vpi_n, {\vmu}_{m:n-1}) +  B_{n,m} - J(\vpi_n, \vf^*_n) \tag{by definition of $\vpi_n^*$}\\
    &= J(\vpi_n, {\vmu}_{m:n-1})  - J(\vpi_n, \vf^*_n)  + \lambda_{n,m} \Sigma_{m:n-1}(\vpi_n, {\vmu}_{m:n-1}) +  B_{n,m} \notag\\
    &\leq \lambda_{n,m} \Sigma_{m:n-1}(\vpi_n, \vf^*_n) + \lambda_{n,m} \Sigma_{m:n-1}(\vpi_n, {\vmu}_{m:n-1}) + 2\, B_{n,m} \tag{by \cref{lem:performance_difference}} \\
    &= 2\, \lambda_{n,m} \Sigma_{m:n-1}(\vpi_n, \vf^*_n) + \lambda_{n,m} \left(\Sigma_{m:n-1}(\vpi_n, {\vmu}_{m:n-1}) - \Sigma_{m:n-1}(\vpi_n, \vf^*_n)\right) + 2\, B_{n,m} \tag{add and subtract $\lambda_{n,m} \Sigma_{m:n-1}(\vpi_n, \vf^*_n)$}\\
    &\leq (\tilde\lambda^2_{n,m} + 2\tilde\lambda_{n,m}) \Sigma_{m:n-1}(\vpi_n, \vf^*_n) + (\tilde\lambda_{n,m} + 2)\,\tilde B_{n,m}. \notag
\end{align}
In the last line, we used the fact that $\norm{{\vsigma}_{m:n-1}(\cdot,\cdot)}$ is bounded by $\sigma_{\max}$ and positive, and thus behaves like a reward. In fact, it is an intrinsic reward~\cite{sukhija_optimism_2025}, which is why \cref{lem:performance_difference} also applies. Thus, the epistemic uncertainty scaling $\lambda_{n,m}$ and drift accumulation $B_{n,m}$ from (\ref{eq: exploration bonus weight},\ref{eq: temporal drift accumulation}) are modified to become $\tilde \lambda_{n,m}$ and $\tilde B_{n,m}$ which include the term $\max\{R_{\max}, \sigma_{\max}\}$.
\end{proof}
\end{lemma}
\pagebreak
\begin{proof}[Proof of \cref{th:main_regret}]
We now prove our main regret result in terms of the total variation $P_N$, the maximum information gain over $p$ episodes $\gamma_{p}$ and the horizon $N$. The derivation is analogous for the resetting and sliding window case with only $\beta_n(\delta,\cdot)$ behaving differently w.r.t.~$n$ as defined in \cref{lem:calibration}.

At episode $n$, let $m(n)$ denote the first episode whose data is retained in the model buffer used to construct the statistical model at episode $n$. Thus, the model is fit on data from episodes $m(n),\,m(n)+1,\,\dots,\,n-1$.

We denote by $p$ the maximal buffer length, i.e.,
\[
n-m(n)\le p
\qquad\text{for all } n\in\{1,\dots,N\}.
\]
In the resetting scheme, $m(n)=n_0(n)$ and $p=H$. In the sliding-window scheme, $m(n)=n-w$ and $p=w$.
Hence, in both cases the model is fit on at most $p$ episodes of data. For notational simplicity, we suppress the dependence of $m$ on $n$ whenever this is clear from context.

Let $r_{n} \defeq J(\vpi_n^*, \vf_n^*) - J(\vpi_n, \vf_n^*)$ denote the (episodic) dynamic regret.
The success event of \cref{lem:calibration} holds with probability at least $1-\delta$. Conditioned on this success event it holds by \cref{lem:tv_episodic_regret_bound} that for all $n\ge 1$,
\[
r_{n} \;\le\; (\tilde\lambda_{n,m}^2+2\tilde\lambda_{n,m})\,\Sigma_{m:n-1}(\vpi_n,\vf_n^*)
\;+\;
(\tilde\lambda_{n,m} + 2)\,\tilde B_{n,m}.
\]
The cumulative regret is then

\begin{align}
\Rdyn_N
&\defeq \sum_{n=1}^{N} r_n \notag\\
&\le
\sum_{n=1}^{N}
\bigl(\tilde\lambda_{n,m}^2 + 2\tilde\lambda_{n,m}\bigr)\,
\Sigma_{m:n-1}(\vpi_n,\vf_n^*)
+
\sum_{n=1}^{N}
\bigl(\tilde\lambda_{n,m}+2\bigr)\,
\tilde B_{n,m}.
\label{eq:sum_decomposition1}
\end{align}

Note that for all buffer-update schemes considered in this work, the statistical model at episode $n$ is fit on at most $p$ episodes of data. Moreover, both $\beta_n(\delta,n-m)$ and $\xi_{n-m}$ from \cref{lem:calibration} are nondecreasing in the buffer length $n-m$. Therefore, there exists a uniform upper bound on the optimism coefficient such that
\[
\tilde\lambda_{n,m}
\le
\frac{\max\{R_{\max},\sigma_{\max}\}T}{\sigma}\,\beta_N(\delta,p)
\;\defeq\;
\bar{\lambda}_{N,p},
\qquad
\xi_{n-m}\le \xi_p.
\]
For the resetting scheme, one may take $\beta_N(\delta,p)=\beta_N(\delta,H)$ with $p=H$, while for the sliding-window scheme one takes $\beta_N(\delta,p)=\beta_N(\delta,w)$ with $p=w$. The only difference between the two cases is thus the precise form of the confidence parameter $\beta_N(\delta,p)$ from \Cref{lem:calibration}; the remainder of the argument is identical.

Therefore, substituting the definition of $\tilde B_{n,m}$ from \cref{lem:tv_episodic_regret_bound} into \eqref{eq:sum_decomposition1}, it holds that
\begin{align}
\Rdyn_N
&\le
(\bar\lambda_{N,p}^2 + 2\bar\lambda_{N,p})
\sum_{n=1}^{N} \Sigma_{m:n-1}(\vpi_n,\vf_n^*)
\notag\\
&\quad+
(\bar\lambda_{N,p}+2)\,
\frac{\max\{R_{\max},\sigma_{\max}\}T^2}{\sigma}\,\xi_p
\sum_{n=1}^{N}
\sum_{s=m}^{n-1}
\norm{\vf_{s+1}^*-\vf_s^*}_{\mathcal H_k^{d_x}}. \label{eq:regret decomp}
\end{align}

We first bound the last term involving the double sum. Fix some $s \in \{1,\dots,N-1\}$. The increment $\|\vf_{s+1}^*-\vf_s^*\|_{\mathcal H_k^{d_x}}$ appears in the inner sum for exactly those episodes $n$ such that $s \in \{m(n),\dots,n-1\}$, that is, for those episodes whose model buffer contains episode $s$. Since every buffer-update scheme considered in this work retains at most $p$ episodes, there are at most $p$ such episodes. Therefore each drift increment is counted at most $p$ times, and hence
\begin{align*}
\sum_{n=1}^{N}\sum_{s=m}^{n-1} \|\vf_{s+1}^*-\vf_s^*\|_{\mathcal H_k^{d_x}}
&\le
p \sum_{s=1}^{N-1} \|\vf_{s+1}^*-\vf_s^*\|_{\mathcal H_k^{d_x}} \\
&\le p P_N. \tag{by \cref{ass:variation_budget}}
\end{align*}

Next, we bound the first term involving the cumulative epistemic uncertainty $\Sigma_{m(n):n-1}(\cdot,\cdot)$. 
We decompose the sum over the $\sfrac{N}{p}$ groups of size $p$. For group
$q \in \{0,\dots,\sfrac{N}{p}-1\}$, the episodes are $n=qp+1,\dots,(q+1)p$. Writing
$$n=qp+h+1,\qquad h=0,\dots,p-1,$$ and defining $m(q,h)\defeq m(qp+h+1)$, we obtain
\begin{align}
    \sum_{n=1}^{N} \Sigma_{m(n):n-1}(\vpi_n,\vf_n^*) 
    &= \sum_{q=0}^{\sfrac{N}{p}-1}\sum_{h=0}^{p-1} \Sigma_{m(q,h):qp+h}\bigl(\vpi_{qp+h+1},\vf_{qp+h+1}^*\bigr) \notag \\
        &\le
    \sum_{q=0}^{\sfrac{N}{p}-1}
    \sqrt{p}
    \left(
        \sum_{h=0}^{p-1}
        \Sigma_{m(q,h):qp+h}^2\bigl(\vpi_{qp+h+1},\vf_{qp+h+1}^*\bigr)
    \right)^{1/2}, \label{eq:squared_uncertainty}
\end{align}
where the last step follows from the Cauchy-Schwarz inequality.
It remains to bound the squared uncertainty terms inside each group. Recall the definiton of $\Sigma_{m:\ell}(\cdot,\cdot)$ from \eqref{eq:sigma_trajectory}. Applying Cauchy-Schwarz again in the time index $t$ gives

\begin{align*}
\Sigma_{m:\ell}(\vpi,\vf)
&=
\E^{\vpi, \vf}_{\vw_{0:T-1}}\!\left[
\sum_{t=0}^{T-1}
\norm{\vsigma_{m:\ell}(\vx_{t},\vpi(\vx_{t}))}_2
\right]\\
&\leq 
\left( T
\E^{\vpi, \vf}_{\vw_{0:T-1}}\!\left[
\sum_{t=0}^{T-1}
\norm{\vsigma_{m:\ell}(\vx_{t},\vpi(\vx_{t}))}_2^2
\right] \right)^{1/2}
\end{align*}

Therefore,
\begin{align*}
\Sigma_{m:\ell}^2(\vpi,\vf)
\le
T\,
\E^{\vpi,\vf}_{\vw_{0:T-1}}\!\left[
\sum_{t=0}^{T-1}
\norm{\vsigma_{m:\ell}(\vx_{t},\vpi(\vx_{t}))}_2^2
\right].
\end{align*}

Summing over $h=0,\dots,p-1$ within one group,
\begin{align*}
    \sum_{h=0}^{p-1}
    \Sigma_{m(q,h):qp+h}^2\bigl(\vpi_{qp+h+1},\vf_{qp+h+1}^*\bigr) 
    \le
    T \sum_{h=0}^{p-1}
    \E_{\vw_{0:T-1}}^{\vpi_{qp+h+1},\vf_{qp+h+1}^*}\!\left[
        \sum_{t=0}^{T-1}
        \left\|
            \vsigma_{m(q,h):qp+h}\bigl(\vx_t,\vpi_{qp+h+1}(\vx_t)\bigr)
        \right\|_2^2
    \right].
\end{align*}

Now, we use the GP information-gain bound on the cumulative posterior variance over one such group. By Lemmas 14--17 of~\cite{curi2020efficient}, applied componentwise to the $d_x$ output dimensions, there exists a constant $C_\sigma>0$ such that for any indices $m \le \ell \le m+p-1$ and any sequence of query points $z_{\ell,t}$, 
\[
\sum_{\ell=m}^{m+p-1}\sum_{t=0}^{T-1}
\bigl\|\vsigma_{m:\ell-1}(\vz_{\ell,t})\bigr\|_2^2
\le C_\sigma T \gamma_p,
\]
where $\gamma_p$ denotes the maximum information gain associated with a data buffer of length $p$ as defined in \eqref{eq:max_info_gain}.
We use this bound in our groupwise form: there exists a constant $C_\sigma>0$ such that for every group $q$,
\begin{align*}
\sum_{h=0}^{p-1}
\E_{\vw_{0:T-1}}^{\vpi_{qp+h+1},\vf_{qp+h+1}^*}\!\left[
\sum_{t=0}^{T-1}
\left\|
\vsigma_{m(q,h):qp+h}\bigl(\vx_t,\vpi_{qp+h+1}(\vx_t)\bigr)
\right\|_2^2
\right]
\le
C_\sigma T \gamma_p.
\end{align*}

Even though the model buffer $m(q,h):qp+h$ may change with every episode $h$ within a group $q$, each term above is the posterior variance of a GP fit on at most $p$ episodes of data, and the cumulative posterior variance over any group of $p$ consecutive prediction episodes is controlled by the maximum information gain associated with a buffer of length $p$. Therefore,
\begin{align*}
    \sum_{h=0}^{p-1}
    \Sigma_{m(q,h):qp+h}^2\bigl(\vpi_{qp+h+1},\vf_{qp+h+1}^*\bigr)
    \le
    C_\sigma T \gamma_p.
\end{align*}

Combining this bound with the groupwise decomposition \eqref{eq:squared_uncertainty} above yields
\begin{align*}
    \sum_{n=1}^{N} \Sigma_{m(n):n-1}(\vpi_n,\vf_n^*) 
    &\le
    \sum_{q=0}^{\sfrac{N}{p}-1}
    \sqrt{p}\,\sqrt{C_\sigma T^2 \gamma_p} \\
    &= \frac{N}{p}\sqrt{p\,C_\sigma T^2\,\gamma_p} \\
    &= N \sqrt{\frac{C_\sigma T^2\,\gamma_p}{p}}.
\end{align*}
Substituting this bound into the cumulative regret decomposition \eqref{eq:regret decomp} yields
\begin{align*}
    \Rdyn_N
    &\le
    (\bar\lambda_{N,p}^2+2\bar\lambda_{N,p})\,
    N \sqrt{\frac{C_\sigma T^2\,\gamma_p}{p}}
    +
    (\bar\lambda_{N,p}+2)\,
    \frac{\max\{R_{\max},\sigma_{\max}\}T^2}{\sigma}\,
    \xi_p\,p P_N.
\end{align*}

We now specialize this generic bound to the two buffer-update schemes considered in this work.

For the resetting scheme, we have $p=H$, $\bar\lambda_{N,p}=\tilde\lambda_{n,H}= \mathcal O(\sqrt{\gamma_H})$, and $\xi_p=\xi_H= \mathcal O(\sqrt{H\gamma_H})$. Hence,
\[
\Rdyn_N
=
\mathcal O\!\left(
N\sqrt{\frac{\gamma_H^3}{H}}
+
\gamma_H H^{3/2} P_N
\right).
\]

For the sliding-window scheme, we have $p=w$, $\bar\lambda_{N,p}=\bar\lambda_{N,w}$, and $\xi_p=\xi_w=\mathcal O(\sqrt{w\gamma_w})$, with
\begin{align*}
\bar{\lambda}_{N,w}
& =
\frac{\max\{R_{\max},\sigma_{\max}\}T}{\sigma}
\Bigl(
B+\sigma\sqrt{2\bigl(\gamma_w+1+d_x\log(NT/\delta)\bigr)}
\Bigr),\\
& = \mathcal O(\sqrt{\gamma_w+\log(NT/\delta)}).
\end{align*}
Therefore,
\[
\Rdyn_N
=
\mathcal O\!\left(
N\sqrt{\frac{\gamma_w\bigl(\gamma_w+\log(NT/\delta)\bigr)^2}{w}}
+
\gamma_w w^{3/2} P_N
\right)
=
\tilde{\mathcal O}\!\left(
N\sqrt{\frac{\gamma_w^3}{w}}
+
\gamma_w w^{3/2} P_N
\right).
\]

Thus, in both cases,
\[
\Rdyn_N
=
\tilde{\mathcal O}\!\left(
N\sqrt{\frac{\gamma_p^3}{p}}
+
\gamma_p p^{3/2} P_N
\right),
\]
where $p=H$ for resetting and $p=w$ for the sliding-window scheme. The only difference between the two cases is the additional logarithmic factor in the confidence parameter $\beta_n(\delta,\cdot)$ for the sliding-window scheme.
\end{proof}

\section{Additional Experimental Details}

\subsection{Implementation Details}
\label{app:experiments}
Our code is publicly available.\footnote{Repository link: \url{https://github.com/lasgroup/ombrl}}
For experiments that do not explicitly use Gaussian Processes, we model the forward dynamics using an ensemble of 5 neural networks. 
Epistemic uncertainty is estimated via the disagreement among ensemble members~\cite{lakshminarayanan_ensembles_2017}. 
The policy is optimized using Soft Actor-Critic (SAC)~\cite{haarnoja2018soft}.
The full set of hyperparameters for the statistical model and SAC are listed in~\cref{tab:sac}.

For the ensemble-based experiments, we consider environments from the Gym and MuJoCo benchmark suites~\cite{brockman_gym_2016, todorov_mujoco_2012}, which we adapt to the non-stationary setting by introducing time-varying actuator limits, as described in Appendix \ref{app:non-stationarity}.
In \cref{fig:decay_rates} and \ref{fig:different_envs}, we additionally auto-tune the intrinsic reward weight $\lambda_{n,m}$ from \eqref{eq: exploration bonus weight} following thr auto-tuning approach from \cite{sukhija2025maxinforl}, who show that this approach performs robustly across a range of off-policy RL methods.
\begin{table}[ht]
\centering
\caption{Hyperparameters for ensemble-based experiments with SAC in~\Cref{sec:experiments}.}
\label{tab:sac}
\begin{adjustbox}{max width=\linewidth}
\begin{threeparttable}
\begin{tabular}{l|cccccc}
\toprule
{Environment} &
{Action Repeat} &
{Policy / Critic Arch.} &
{Model Arch.} &
{$H$ / $w$ [episodes]} &
{Learning Rate} &
{Batch Size} \\
\midrule
Pendulum                  & 1 & (256,256) & $5\times$(256,256) & 20   & $3 \times 10^{-4}$ & 256 \\
InvertedDoublePendulum    & 1 & (256,256) & $5\times$(256,256) & 150  & $3 \times 10^{-4}$ & 256 \\
HalfCheetah               & 2 & (256,256) & $5\times$(256,256) & 200  & $3 \times 10^{-4}$ & 256 \\
Hopper                    & 2 & (256,256) & $5\times$(256,256) & 200  & $3 \times 10^{-4}$ & 256 \\
Walker2d                  & 2 & (256,256) & $5\times$(256,256) & 1200 & $3 \times 10^{-4}$ & 256 \\
Ant                       & 2 & (256,256) & $5\times$(256,256) & 250  & $3 \times 10^{-4}$ & 256 \\
Humanoid                  & 2 & (512,512) & $5\times$(512,512) & 1000 & $3 \times 10^{-4}$ & 256 \\
HumanoidStandup           & 2 & (256,256) & $5\times$(256,256) & 250  & $3 \times 10^{-4}$ & 256 \\
\bottomrule
\end{tabular}
\begin{tablenotes}
\footnotesize
\item $H$ denotes the reset period for R-OMBRL and $w$ the window size for SW-OMBRL. 
The window is implemented via a replay buffer of fixed size in transitions, corresponding approximately to $w$ recent episodes.
\end{tablenotes}
\end{threeparttable}
\end{adjustbox}
\end{table}

\subsection{Non-Stationarity Implementation}
\label{app:non-stationarity}

We induce time-varying dynamics by modifying the actuator strength of each environment through an episode-dependent schedule. 
In all experiments, we use an exponential decay of a control-related parameter, such as the maximum torque or actuator gain. 
Concretely, at episode $n$, the parameter is updated according to
\begin{equation}
    \theta_n = \exp(-a \, \max(0, n - n_{\text{start}}))\,(\theta_{\max} - \theta_{\min}) + \theta_{\min},
\end{equation}
where $a \geq 0$ is the decay rate and $n_{\text{start}}$ denotes the episode at which the decay begins. 

The parameter $\theta_n$ is applied directly to the environment at the beginning of each episode by modifying the actuator limits or scaling the underlying dynamics parameters. 
For example, in the Pendulum environment, we adjust the maximum torque, while in MuJoCo environments, we scale actuator gear ratios, effectively changing the available control authority.
Table~\ref{tab:decay_params} summarizes the parameters used for each environment.
\begin{table}[t]
\centering
\caption{Parameters for the exponential non-stationarity schedule used in the experiments.}
\label{tab:decay_params}
\begin{adjustbox}{max width=\linewidth}
\begin{tabular}{l|ccccccc}
\toprule
& \multicolumn{3}{c}{\cref{fig:decay_rates}} & \multicolumn{1}{c}{\cref{fig:different_envs}} \\
\cmidrule(lr){2-4} \cmidrule(lr){5-5}
Environment & $a$ (slow decay) & $a$ (medium decay) & $a$ (fast decay) & $a$ 
& $\theta_{\max}$ & $\theta_{\min}$ & $n_{\text{start}}$ \\
\midrule
Pendulum                & 0.025 & 0.05  & 0.1   & --        & 5.0 & 1.0  & 5    \\
HalfCheetah             & 0.002 & 0.005 & 0.007 & --        & 1.0 & 0.2  & 400  \\
Hopper                  & 0.002 & 0.005 & 0.007 & --        & 1.0 & 0.3  & 300  \\
\midrule
InvertedDoublePendulum  & --    & --    & --    & 0.02      & 1.0 & 0.25 & 200  \\
Ant                     & --    & --    & --    & 0.007     & 1.0 & 0.3  & 400  \\
Walker2d                & --    & --    & --    & 0.0005    & 1.0 & 0.3  & 1500 \\
HumanoidStandup         & --    & --    & --    & 0.02      & 1.0 & 0.3  & 300  \\
Humanoid                & --    & --    & --    & $3\times10^{-6}$ & 1.0 & 0.3  & 2000 \\
\bottomrule
\end{tabular}
\end{adjustbox}
\end{table}

\subsection{Effect of Soft Resets}
\label{app:alpha}

Beyond replay-buffer adaptation, R-OMBRL and SW-OMBRL also incorporate \emph{soft resets} of the dynamics model and the policy. Concretely, the model parameters $\bm{\phi}$ and the policy parameters $\bm{\theta}$ are periodically perturbed toward their initialization, as described in \eqref{eq:soft_reset}, controlled by the perturbation rates $\alpha_1$ for the dynamics model and $\alpha_2$ for the policy. These parameters govern the trade-off between retaining previously learned structure and enabling rapid adaptation to the current dynamics.

\cref{fig:ablate_alpha} studies this trade-off for R-OMBRL under non-stationary dynamics. Overall, the results show that this perturbation rate has a strong effect on performance. In particular, aggressive perturbations degrade performance noticeably, especially in the transient phase after a replay-buffer reset. This suggests that even under changing dynamics, the policy benefits from preserving useful structure instead of being repeatedly pushed too far toward reinitialization.

We find that varying the model perturbation rate leads to comparatively smaller changes in performance than the policy perturbation rate, since model errors affect control only indirectly, whereas overly strong policy perturbations immediately impair action selection. Across environments, moderate perturbation rates provide the best trade-off between stability and plasticity, with $\alpha_1=\alpha_2=0.2$ yielding robust performance and therefore used in the main experiments.

\begin{figure}[ht]
    \centering
    \includegraphics[width=0.6\linewidth]{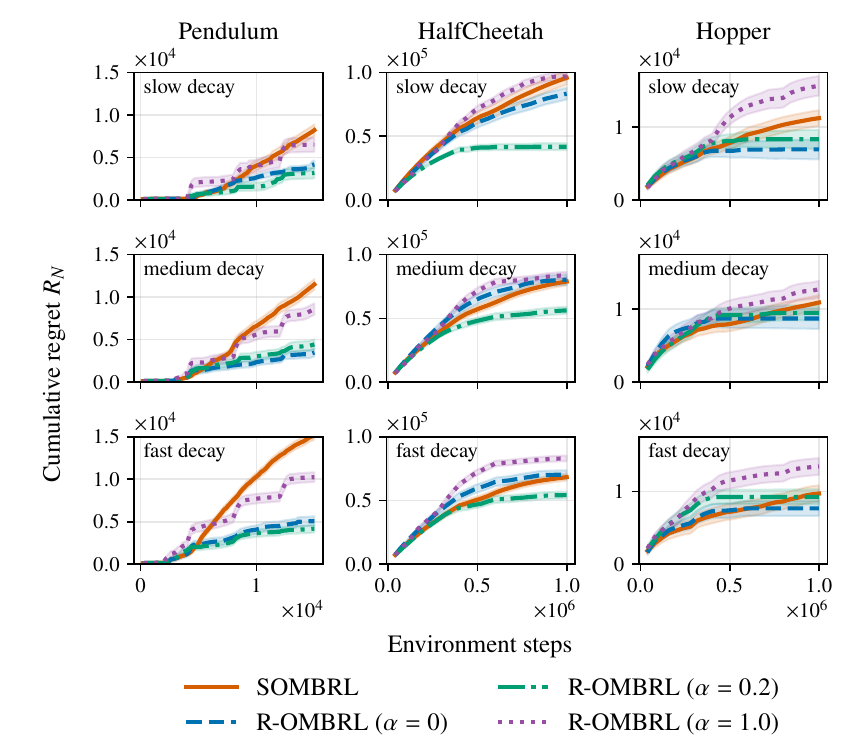}
\caption{Ablation of the soft reset parameter $\alpha_1=\alpha_2=\alpha$ from \eqref{eq:soft_reset} for R-OMBRL under non-stationary dynamics. We vary the strength of the parameter perturbations applied to the dynamics model and policy after replay-buffer resets. Moderate perturbations yield the best trade-off between stability and plasticity, while aggressive policy perturbations are particularly detrimental.}    \label{fig:ablate_alpha}
\end{figure}

\section{Hardware Details}
\label{app:hardware_details}

We follow the experimental setup of~\cite{rothfuss2024bridging} and use the same training configuration for both SOMBRL and R-OMBRL. The only difference between the methods lies in the replay buffer: SOMBRL uses a growing buffer, while R-OMBRL periodically resets the buffer. We do not employ soft resets in these experiments, i.e., $\alpha_1 = \alpha_2 = 0$.

To induce non-stationarity, we modify the maximum throttle parameter. After episode $N=14$, the maximum throttle is decayed exponentially to $25\%$ of its initial value over the course of three episodes. This results in a rapid transition from a dynamic slide-parking task to a standard parking task. To isolate the core behhavior of replay-buffer adaptation, we use a reset period of $H = 15$.

We provide additional qualitative results in the accompanying videos.\footnote{Video link: \url{https://polybox.ethz.ch/index.php/s/2FToMjCmTWX4j5w}}
The stationary baseline is evaluated under fixed dynamics with both high and low maximum throttle.

For the non-stationary setting, we visualize the behavior of SOMBRL and R-OMBRL under the same rapidly decaying throttle schedule. For SOMBRL, we show (i) the policy trained under high throttle before the change, (ii) the behavior immediately after the dynamics change, where the policy still attempts to drift, (iii) a representative failure case, and (iv) the best-performing policy across all seeds. For R-OMBRL, we show (i) the policy before the change, (ii) the behavior shortly after the dynamics change and a recent reset, (iii) a near-converged policy exhibiting stable parking without drifting, and (iv) the best-performing policy across seeds.

\end{bibunit}

\stopcontents[sections]
\else
\pagebreak
\end{bibunit}
\fi

\end{document}